\def\arxiv{arxiv}

\ifx\arxiv\undefined
\documentclass[twocolumn]{article}
\else
\documentclass[]{article}
\fi

\usepackage[width=17cm,height=22cm]{geometry}
\usepackage[english]{babel}
\usepackage[utf8]{inputenc}
\usepackage{fancyvrb}
\usepackage{authblk}
\usepackage{amsmath}
\usepackage{amsthm}
\usepackage{amssymb}
\usepackage{textcomp}
\usepackage{hyperref}
\usepackage{multirow}

\newtheorem{theorem}{Theorem}
\newtheorem{definition}{Definition}
\newtheorem{lemma}{Lemma}

\ifx\arxiv\undefined
\title{Learning from networked examples in a $k$-partite graph\footnote{A longer version of this paper appears at http://arxiv.org/abs/1306.0393 [arXiv:1306.0393].}}
\else
\title{Learning from networked examples in a $k$-partite graph}
\fi
\author[1]{Yuyi Wang}
\author[1]{Jan Ramon}
\author[2]{Zheng-Chu Guo}
\affil[1]{KULeuven, Belgium}
\affil[2]{University of Exeter, UK}

\begin{document}
\maketitle

\begin{abstract}
Many machine learning algorithms are based on the assumption that training examples are drawn independently.
However, this assumption does not hold anymore when learning from a networked examples, i.e. examples sharing pieces of information (such as vertices or edges).
We propose an efficient weighting method for learning from networked examples
and show a sample error bound which is better than previous work.
\end{abstract}

\medskip


\section{Introduction}
In supervised learning, a labeled training sample for learning 
takes the form
$\mathbf{Z}=\{{\bf z}_i\}_{i=1}^m$
with ${\bf z}_i=(\mathbf{x}_i , y_i)\in(\mathcal{X}\times \mathcal{Y})$
where ${\cal X}$ is called the feature space or the input space,
and ${\cal Y}$ is called the label space or the output space.
A standard assumption is that the training examples ${\bf z}_i$ from $\mathbf{Z}$ are drawn independently and identically (i.i.d.) from
a probability distribution
$\rho$ on ${\cal Z}= {\cal X} \times {\cal Y}$.

In this paper, we consider a setting where examples share part of their features.
The i.i.d. assumption does not always hold in this setting.  
For instance, suppose that we are interested in predicting 
whether a given person likes a given movie.  
We could ask a set of persons to grade 
five of the movies they have seen in the past.   
Then, we want to predict for a new visitor (drawn from the same distribution as our training
persons)  whether he will like a newly introduced movie (having features
drawn from the same distribution as the movies in the past).  
Our training examples (containing a person ID, a movie ID and a grade)
are not all independent since each person graded 
several movies and all movies were graded by several persons.  
Still, we would like to get a generalization guarantee.

A naive method would be to ignore the problem and treat the training examples as independent examples. 
The result in \cite{janson04} showed that a larger number of possibly non-independent examples 
would not necessarily mean that a more accurate model can be constructed.

Another straightforward method is to
first find a subset of the training examples which are independent
and then learn from these.
Though we can directly use existing results to bound the sample error of this approach,
it is inherently difficult to find a sufficiently large independent set of training examples.  
Moreover, we will show that possibly not all available information is used and the solution is suboptimal.

In this paper, we propose a novel approach to learn from networked examples.
In our method, we first compute nonnegative weights for all training examples.
Using these weighted examples, we show that we can get better bounds of the sample error than the two methods above.
It is an advantage that the weights of the examples are efficiently computable.

The remainder of this paper is organized as follows.
We introduce networked training examples in Section \ref{sec:problem}.
We review some basic concepts of statistical learning theory in Section \ref{sec:pre}.
The related work is discussed in Section \ref{sec:related}, and this section mainly gives the sample error bounds of learning from networked data if we treat the data as i.i.d.. 
In Section \ref{sec:mis}, we consider the above-mentioned method that first select a set of independent training examples.
In Section \ref{sec:weighting}, we propose our network learning method. 
We derive weighted inequalities in Section \ref{subsec:inequalities},
and they are used in Section \ref{subsec:ermapproach} to estimate the sample error of the ERM algorithms with networked training examples.
Section \ref{sec:conclusion} concludes this paper with a summary of our contributions and a discussion of future work.

\ifx\arxiv\undefined 

\else 

In table \ref{tbl:notations}, we list the notations used in this paper.

\begin{table}[!hbp]\label{tbl:notations}
\centering
\begin{tabular}{|c |c|}
\hline
 Notation & Meaning  \\
\hline
 $L(\cdot,\cdot)$&  loss function \\ 
 \hline
${\cal H}$&  hypothesis space  \\
\hline
 $f_\rho^L$&  global minimizer of ${\cal{E}}^L(\cdot)$  \\
\hline
 $f_{\cal H}^L$&  minimizer of ${\cal{E}}^L(\cdot)$ in $\cal{H}$ \\
\hline
$f_{\bf Z}^L$&  global minimizer of ${\cal{E}}_{\bf Z}^L(\cdot)$  \\
\hline
$f_{{\bf Z}, {\cal H}}^L$&  minimizer of ${\cal{E}}_{\bf Z}^L(\cdot)$ in ${\cal H}$  \\
\hline
$f_{{\bf Z}}$&  minimizer of ${\cal{E}}_{\bf Z}(\cdot)$ in ${\cal H}$  \\
\hline
$f_{{\bf Z}_\mathsf{s}}$&  minimizer of ${\cal{E}}_{\mathsf{s}}(\cdot)$ in ${\cal H}$  \\
\hline
$f_{\cal H}$&  minimizer of ${\cal{E}}(\cdot)$ in ${\cal H}$  \\
\hline
 ${\cal{E}}^L(\cdot)$&  expected risk w.r.t $L$ \\
\hline
 ${\cal{E}(\cdot)}$&  expected risk w.r.t least square loss \\
\hline
 ${\cal{E}}^L_{\bf Z}(\cdot)$&  empirical risk w.r.t $L$   \\
\hline
 ${\cal{E}}_{\bf Z}(\cdot)$&  empirical risk w.r.t least square loss  \\
\hline
 ${\cal{E}}_{\mathsf{s}}(\cdot)$&  empirical risk w.r.t least square loss and ${\bf Z}_\mathsf{s}$ \\
\hline
${\cal{E}}_{\cal H}(\cdot)$ &  sample error \\
\hline
$G$ &  k-partite hypergraph \\
\hline
$V$& vertices set\\
\hline
$V^{(i)}$& partition $i$ of the vertices\\
\hline
$v^{(i,j)}$& $j$-th vertice in partition $i$\\
\hline
$e$ &  hyperedge \\
 \hline
$E$ & hyperedge set\\
\hline
$\Gamma$&  dependency graph  \\
 \hline
$\alpha(\Gamma)$&  independence number of $\Gamma$  \\
\hline
 $\chi^*(\Gamma)$&  fractional chromatic number of $\Gamma$  \\
\hline
${\cal{N}}(S,\tau)$ &  covering number of the metric space $S$ with radius $\tau$ \\
\hline
$\phi$ &  feature map \\
\hline
  ${\rho}$&  probability distribution  \\
\hline
  ${\mathsf{s}}$&  sum of the optimal weighting ${\bf w}$  \\
\hline
 ${\bf w}=(w_1,\cdots,w_n)$&  non-negative weight  \\
\hline
 ${\cal X}$&  input space  \\
\hline
 ${\cal Y}$&  output space  \\
\hline
 ${\cal Z}$&  ${\cal X}\times {\cal Y}$ \\
\hline
 ${\bf Z}$&  training sample  \\
 \hline
 ${\bf Z}_I$&  maximum independent set of ${\bf Z}$ \\
 \hline
 ${\bf Z}_{\mathsf{s}}$&  weighted training sample  \\
 \hline
\end{tabular}
\caption{Notations}
\end{table}

\fi

\section{Problem Statement}\label{sec:problem}
Before discussing our method, we first give a formal problem statement.

\subsection{The network}\label{subsec:networked}
In this paper, we use a $k$-partite hypergraph $G=(V, E, {\cal X}, {\cal Y},\phi)$ to represent the network which induces all the training examples.
The set of vertices $V$ is partitioned into $k$ disjoint sets $V^{(1)}, V^{(2)}, \ldots, V^{(k)}$,
and each hyperedge $e\in E \subseteq V^{(1)} \times V^{(2)} \times \cdots \times V^{(k)}$ intersects every set of the partition in exact one vertex.
The number of hyperedges is denoted by $m$, and the cardinality of a partition $V^{(i)}$ is denoted by $n_i$, i.e., $|E|=m \text{ and } |V^{(i)}| = n_i$.
The $j$-th vertex of the partition $V^{(i)}$ is denoted by $v^{(i,j)}$ where $1\le j\le n_i$.
We denote the $i$-th component of an edge $e$ as $e^{(i)}$, which is a vertex in $V^{(i)}$.
Two edges $e_a$ and $e_b$ \emph{overlap} if and only if there exists $1 \le i \le k$ such that $e_a^{(i)} = e_b^{(i)}$.

For instance, in our movie rating example we would have a vertex set $V^{(1)}$ of movies, a vertex set $V^{(2)}$ of persons (watching movies) and a set $V^{(3)}$ of movie ratings.  Hyperedges would be triple $(m,p,r)\in V^{(1)}\times V^{(2)}\times V^{(3)}$ of a movie, a person and the rating this person gave to that movie.

\subsection{Features}
Let ${\cal X} = {\cal X}^{(1)} \times \cdots \times {\cal X}^{(k)}$ be a $k$-dimensional compact metric space.
Let $\phi: \bigcup_{1\le i\le k} (V^{(i)} \mapsto {\cal X}^{(i)})$ be a function on the vertex set $V$ assigning to every vertex $v^{(i,j)}$ in $V^{(i)}$ 
a feature $\phi(v^{(i,j)}) = x^{(i,j)}$ drawn independently and identically from a \emph{fixed but unknown} distribution $\rho_{i}$. 
We will call $\phi(v^{(i,j)})$ a feature, 
even though it may be a compound object such as a vector.
We also use the notation $\phi$ as a function on hyperedges.
For any hyperedge $e$, we call $\phi(e) = [\phi(e^{(1)}), \phi(e^{(2)}), \ldots, \phi(e^{(k)})]$ the feature vector of $e$.

For instance, in our movie rating example, $\phi$ could assign to movies $m\in V^{(1)}$ pairs $(genre,length)$, to persons $p\in V^{(2)}$ a triple $(gender, age, nationality)$ and to a rating $r\in V^{(3)}$ a pair $(watching\_time, movie\_version)$.  $\phi$ would therefore assign to every hyperedge a triple containing in total $8$ values.

\subsection{Examples}
Every hyperedge $e_i$ in $E$ induces an example $\mathbf{z}_i = (\mathbf{x}_i, y_i) \in {\cal Z} = {\cal X}\times {\cal Y}$.
The feature vector of this example is $\mathbf{x}_i = \phi(e_i)$.
We will use $\mathbf{x}_i^{(j)}$ to denote the $j$-th component of the feature vector $\mathbf{x}_i$.
If $e_i^{(j)} = v^{(j,l)}$ where $1\le l \le n_j$, then the $j$-th component of the feature of the training example $\mathbf{z}_i$ is $x^{(j,l)}$.
Thus, $\mathbf{x}_i^{(j)} = x^{(j,l)} = \phi(v^{(j,l)})$.
We can see that if two hyperedges overlap, then the two corresponding examples are not independent (they share part of their features, and hence drawing the one example puts restrictions on the drawing of the other example).
Given the features ${\bf x}_i$ of this example, the label $y_i$ follows a fixed but unknown probability distribution $\rho_{y|\mathbf{x}}$.
We can then write $\rho(\mathbf{x},y)=\rho_{y|\mathbf{x}}(\mathbf{x},y)\rho_{\mathbf{x}}(\mathbf{x})$.
The training dataset derived from $G$ is denoted by $\mathbf{Z}=\{\mathbf{z}_i| e_i\in E\}$, and it is called a \emph{$G$-networked sample}.
The size of the sample $\mathbf{Z}$ is the same as the number of hyperedges, so $|\mathbf{Z}|=m$.

\subsection{Independence assumption}
We make the following assumptions:
\begin{itemize}
\item As in the traditional form of PAC-learning,
the feature of every vertex in the partitions $V_i$ is drawn identically and independently from $\rho_i$.
\item Especially, these features are independent from the edges in which they participate, i.e., $\rho_i(x^{(i,l)}) = \rho_i(x^{(i,l)}|E(G))$.
\item Moreover, all hyperedges (examples) get a target value drawn identically and independently from $\rho_{y|\mathbf{x}}$. Even if the hyperedges share vertices, still there target value is sampled i.i.d. from $\rho_{y|\mathbf{x}}$ based on their (possibly identical) feature vector.
\item One can choose freely which vertices participate in which hyperedges, and which edges belong to the training set and the test set, 
as long as this hyperedge and training set selection process is completely independent from the drawing of features for the vertices and the drawing of target values.
\end{itemize}

From the above assumptions, we can infer that 
$\rho_{\mathbf{x}}(\mathbf{x})= \prod_{i=1}^k \rho_i(\mathbf{x}^{(i)})$.
Our analysis of the sample error holds no matter what the distributions $\rho_i$ and $\rho_{y|\mathbf{x}}$ are, as long as the above assumptions hold.

It is possible that the empirical distribution of the training and/or test set deviate from $\rho$, 
but we will show that we can bound the extent to which this is possible based on the assumptions.

In our movie rating example, it may or may not be realistic that these assumptions hold.  In particular, if ratings are obtained from visitors of a cinema, then probably some visitors will already have a preference and will not choose movies randomly.  On the other hand, if ratings are obtained during an experiment or movie contest where a number of participants or jury members are asked to watch a specific list of movies, one could randomize the movies to increase fairness, and in this way our assumptions would be satisfied.

\section{Preliminaries}\label{sec:pre}
In this section, we review some basic concepts of statistical learning theory when the training sample $\bf Z$ is i.i.d.. 
These concepts will be used in following sections. 

\subsection{Learning task}
The main goal of supervised learning is to learn a function $f:\mathcal{X}\mapsto \mathcal{Y}$
from training examples ${\bf Z}$ to predict a label $y$ of an unseen point ${\bf x}$.
\newcommand{\loss}{L}
For convenience, we assume ${\cal Y} = \mathbb{R}$.
We define a loss function $\loss:\mathcal{Y}\times\mathcal{Y}\mapsto \mathbb{R}_+$ to measure the prediction errors.
The function value $\loss(f({\bf x}),y)$ denotes the local error suffered from the use of $f$ to produce $y$ from ${\bf x}.$
We average the local error over all pairs $({\bf x},y)$ by integrating over $\mathcal{Z}$ with respect to $\rho$.
A natural idea is to find the minimizer of the \emph{expected risk}
$$\mathcal{E}^{\loss}(f)=
 \int_{\mathcal{Z}} \loss(f({\bf x}),y) \rho({\bf x},y)\hbox{d}({\bf x},y).
$$
Then the target function we want to learn is defined as
\begin{equation*}
f^{\loss}_\rho=\arg\min \mathcal{E}^{\loss}(f),
\end{equation*}
where the minimization is taken over the set of all measurable functions.
Unfortunately, the probability distribution $\rho$ is unknown,
$f^\loss_\rho$ can not be computed directly.
If every example in ${\bf Z}$ is independent from each other, by the law of large numbers, as the sample size $m$ tends to infinity,
the \emph{empirical risk}
$$\mathcal{E}_{\bf Z}^{\loss}(f)=\frac{1}{m}\sum_{i=1}^m {\loss}(f({\bf x}_i),y_i)$$
converges to the expected risk $\mathcal{E}^{\loss}(f).$
Then we may get a good candidate $f_{\mathbf{Z}}^{\loss}$ to approximate the target function $f^{\loss}_\rho$, where
\begin{equation*}
f_{\mathbf{Z}}^{\loss}=\arg\min \mathcal{E}_{\mathbf{Z}}^{\loss}(f).
\end{equation*}

\subsection{Empirical risk minimization principle}\label{subsec:erm1}
In order to avoid over-fitting, we will not take the minimization of the empirical risk over all
the measurable functions. The main idea of the empirical risk minimization principle is to find
the minimizer in a properly selected hypothesis space $\mathcal{H}$, i.e.,
\begin{equation*}
f_{{\mathbf{Z}}, \mathcal{H}}^{\loss}=\arg\min_{f\in\mathcal{H}}
\mathcal{E}_{\mathbf{Z}}^{\loss}(f).
\end{equation*}
The hypothesis space $\mathcal{H}$ is usually chosen as a subset of $\mathcal{C}(\mathcal{X})$ which is the
Banach space of continuous functions on a compact metric space $\mathcal{X}$ with the norm
$\|f\|_{\infty}=\sup_{\mathbf{x}\in {\cal X}}|f(\mathbf{x})|.$

The performance of the ERM approach is evaluated in terms of the \emph{excess risk}
$$\mathcal{E}^\loss(f_{\mathbf{Z}, \mathcal{H}}^\loss)-\mathcal{E}^\loss(f_\rho^\loss).$$
If we define
\begin{equation*}
f^\loss_{\mathcal{H}}=\arg\min_{f\in\mathcal{H}} \mathcal{E}^\loss(f),
\end{equation*}
then the excess risk can be decomposed as
$$\mathcal{E}^\loss(f_{\mathbf{Z}, \mathcal{H}}^\loss)-\mathcal{E}^\loss(f_\rho^\loss)=
[\mathcal{E}^\loss(f_{\mathbf{Z}, \mathcal{H}}^\loss)-\mathcal{E}^\loss(f_\mathcal{H}^\loss)]+
[\mathcal{E}^\loss(f_\mathcal{H}^\loss)-\mathcal{E}^\loss(f_\rho^\loss)].$$
We call the first part
$\mathcal{E}^\loss(f_{\mathbf{Z}}^\loss)-\mathcal{E}^\loss(f_\mathcal{H}^\loss)$ the {\it sample
error}, 
the second part $\mathcal{E}^\loss(f_\mathcal{H}^\loss)-\mathcal{E}^\loss(f_\rho^\loss)$ the {\it
approximation error}.
The approximation error is independent of the sample and it is well studied in
\cite{cucker07}.
In this paper, we concentrate on the sample error.

Another challenge about the ERM approach is how to choose a proper hypothesis space.
Intuitively, a small hypothesis space brings a large approximation error,
while large hypothesis space results in over-fitting.
Hence the hypothesis space must be chosen to be not
too large or too small. It is closely related to the bias-variance
problem (see, e.g., Section 1.5 of \cite{cucker07}). In learning theory, 
the complexity of the hypothesis space is usually measured in terms of covering number, 
entropy number, VC-dimension, etc. 
In this paper, we will use the covering numbers defined below to measure the capacity of our hypothesis space ${\mathcal{H}}.$

In this paper, we focus on the ERM approach associated with the least square loss function, that is $\loss(f({\bf x}), y)=(f({\bf x})-y)^2.$
Note that our analysis can easily be extended to general loss functions case.

\subsection{Estimating the sample error}\label{subsec:sampleerror}
For the sake of conciseness, we denote $f_{\mathbf{Z}, \mathcal{H}}^\loss,$ $f_\mathcal{H}^\loss,$ $\mathcal{E}_{\mathbf{Z}}^\loss(f)$
and $\mathcal{E}^\loss(f)$ as $f_{\mathbf{Z}},$ $f_\mathcal{H},$ $\mathcal{E}_{\mathbf{Z}}(f)$ and $\mathcal{E}(f)$ respectively.
Now we are in a position to estimate the sample error
$\mathcal{E}_\mathcal{H}(f_{\mathbf{Z}})$.
The definition of $f_{\mathbf{Z}}$ tells us that $\mathcal{E}_{\mathbf{Z}}(f_{\mathbf{Z}})-\mathcal{E}_{\mathbf{Z}}(f_\mathcal{H})\le 0,$ therefore the sample error can be decomposed as
\begin{eqnarray*}
&&\mathcal{E}_\mathcal{H}(f_{\mathbf{Z}})=\mathcal{E}(f_{\mathbf{Z}})-\mathcal{E}(f_\mathcal{H})
=[\mathcal{E}(f_{\mathbf{Z}})-\mathcal{E}_{\mathbf{Z}}(f_{\mathbf{Z}})]\\
&&+[\mathcal{E}_{\mathbf{Z}}(f_{\mathbf{Z}})-
\mathcal{E}_{\mathbf{Z}}(f_\mathcal{H})]+[\mathcal{E}_{\mathbf{Z}}(f_\mathcal{H})-\mathcal{E}(f_\mathcal{H})]\\
&&\leq [\mathcal{E}(f_{\mathbf{Z}})-\mathcal{E}_{\mathbf{Z}}(f_{\mathbf{Z}})]+[\mathcal{E}_{\mathbf{Z}}(f_\mathcal{H})-\mathcal{E}(f_\mathcal{H})].
\end{eqnarray*}

Before stating the results, we first introduce some notations and definitions.
\begin{definition}\label{def:coveringnumber}
Let $S$ be a metric space and $\tau>0.$
We define the {\it covering number} $\mathcal{N}(S,\tau)$ to be the minimal $\ell\in\mathbb{N}$ such that there exists $\ell$
disks in S with radius $\tau$ covering $S$.
When $S$ is compact, this number is finite.
\end{definition}

\begin{definition}\label{def:Mbounded}
Let $M>0$ and $\rho$ be a probability distribution on ${\cal Z}.$
We say that a set $\cal{H}$ of functions from $\mathcal{X}$ to $\mathcal{Y}$ is {\it M-bounded} when
$$\sup_{f\in\mathcal{H}}|f({\bf x})-y|\leq M$$
holds almost everywhere on $\mathcal{Z}.$
\end{definition}

To bound the sample error, the Bernstein inequality is used \cite{bernstein24}. 

\begin{theorem}\label{thm:iidinequalities}
Let $\mathbf{Z}$ be an i.i.d sample and $\xi$ be a function defined on the space ${\cal Z}$ with mean $\mathbf{E}(\xi) = \mu$,
variance $\sigma^2(\xi) = \sigma^2$ and satisfying $|\xi(\mathbf{z})-\mu|\le M$ for almost all $\mathbf{z} \in {\cal Z}$.  Then for all $\epsilon > 0$,
\begin{align*}
& \Pr\left(\frac{1}{m}\sum_{i} \xi(\mathbf{z}_i) - \mu \ge \epsilon \right)
\le \exp \left( -\frac{m\epsilon^2}{2(\sigma^2+\frac{1}{3}M\epsilon)} \right). \\
\end{align*}
\end{theorem}

We estimate the sample error by the above concentration inequality. 
In this paper, we omit the details of the proof and directly quote the following result from \cite{cucker07}.

\begin{theorem}
Let $\mathcal{H}$ be a compact and convex subset of $\mathcal{C}(\mathcal{X})$.
If $\mathcal{H}$ is M-bounded, then for all $\epsilon>0$,
$$\Pr\big(\mathcal{E}_\mathcal{H}(f_{\mathbf{Z}})\ge \epsilon\big)\le \mathcal{N}\Big(\mathcal{H},\frac{\epsilon}{12M}\Big)  \exp\Big(-\frac{m\epsilon}{300M^4}\Big).$$
\end{theorem}

\section{Related work}\label{sec:related}
In this section, we discuss the related work.

\subsection{Dependency graphs}
As described in \cite{janson04}, a \emph{dependency graph} can be used to represent the relationship between the training examples in $\bf Z$.
The vertices of the dependency graph $\varGamma$ are the hyperedges in $G$, that is, $V(\varGamma) = E(G)$.
Thus, the vertices in the dependency graph also represent training examples in $\bf Z$. 
Two vertices are adjacent if the corresponding two hyperedges overlap, 
i.e., if two hyperedges $e_a$ and $e_b$ in $E(G)$ satisfy that there exists $j$ such that $e_a^{(j)} = e_b^{(j)}$, 
then the two vertices $e_a$ and $e_b$ are adjacent in $\varGamma$ 
and the induced examples ${\bf z}_a$ and ${\bf z}_b$ are not independent. 

\subsection{The chromatic-number bound}
In \cite{janson04}, the author shows an inequality which can be used to bound the error on averaging a function over networked sample.

\begin{theorem}\label{thm:chromatics}
Let $\mathbf{Z}$ be a $G$-networked sample and $\xi$ be a function defined on the space ${\cal Z}$ with mean $\mathbf{E}(\xi) = \mu$,
and satisfying $|\xi(\mathbf{z})-\mu|\le M$ for almost all $\mathbf{z} \in {\cal Z}$.  Then for all $\epsilon > 0$,
\begin{align*}
& \Pr\left(\frac{1}{m}\sum_{i} \xi(\mathbf{z}_i) - \mu \ge \epsilon \right)
\le \exp \left(-\frac{8m\epsilon^2}{25\chi^{*}(\varGamma)(\sigma^2+M\epsilon/3)} \right),\\
\end{align*}
where $\chi^{*}(\varGamma)$ is the fractional chromatic number of the dependency graph.
\end{theorem}

Let us now consider a learning strategy we call EQW (EQual Weight) and which learns from a set of networked examples in the same way as if they were i.i.d. 
(i.e. without weighting them as a function of the network structure).
We can use Theorem \ref{thm:chromatics} above to bound the sample error of EQW:

\begin{theorem}
Let $\mathcal{H}$ be a compact and convex subset of $\mathcal{C}(\mathcal{X})$, and $\mathbf{Z}$ be a $G$-networked sample.
If $\mathcal{H}$ is M-bounded, then for all $\epsilon>0$,
\begin{align*}
&\Pr\big(\mathcal{E}_\mathcal{H}(f_{\mathbf{Z}})\ge \epsilon\big)\le \mathcal{N}\Big(\mathcal{H},\frac{\epsilon}{12M}\Big)  
\exp\Big(-\frac{3m\epsilon}{1400\chi^{*}(\varGamma)M^4}\Big).\\
\end{align*}
\end{theorem}

The result above shows that the bound of the sample error does not only rely on the sample size 
but also the fractional chromatic number of the dependency graph. 
That is, a larger sample may result in a poorer sample error bound since $\chi^{*}(\varGamma)$ can also become larger.

\subsection{Mixing conditions}
There is also some literature on learning from a sequence of examples where examples closeby in the sequence are dependent.
In the community of machine learning, 
mixing conditions are usually used to quantify the dependence of sample points and are usually used in time series analysis.
For example, in \cite{guo11}, the learning performance of a regularized classification algorithm using a non-i.i.d. sample is investigated,
where the independence restriction is relaxed to so-called $\alpha$-mixing or $\beta$-mixing conditions.
In \cite{sun10}, regularized least square regression with dependent samples is considered under the assumption that the training sample satisfies some mixing conditions.
In \cite{modha96}, the authors established a Bernstein type 
inequality is presented for stationary exponentially $\alpha$-mixing processes,
which is based on the effective number (less than the sample size).
Our Bernstein type inequalities for dependent network data 
too assigns weights to examples. 
However, the assumptions for the the training sample are different, and the main techniques are distinct. 
Moreover, in practice, it is not easy to check whether 
the training sample satisfies the mixing conditions. 
Our networked training examples certainly do not satisfy any of these mixing conditions.
We refer interested readers to \cite{bradley05} and references therein for more details about the mixing conditions.

\subsection{Hypothesis tests}
In \cite{wang:correcting}, the authors consider a similar setting of networked examples.  
They also use the dependency graph to represent the examples and their realations. 
While we assume a worst case over all possible dependencies, 
and allow to model explicitely causes of dependencies (represented with vertices which can be incident with more than two edges), this work assumes a bounded covariance between pairs of examples connected with an edge (excluding possible higher-order interactions).
While we use our model to show learning guarantees, 
\cite{wang:correcting} shows corrections for the bias (induced by the dependencies between examples) on statistical hypothesis tests.  
It seems plausible that both models can be applied for both the learning guarantee and statistical testing tasks.

\section{Selecting an independent subset of training examples}\label{sec:mis}
A straightforward idea to learn from a $G$-networked sample $\mathbf{Z}$ is to
find a subset $\mathbf{Z}_I \subseteq \mathbf{Z}$ of training examples which 
correspond to non-overlapping hyperedges.
Due to our assumptions, such set will be an i.i.d. sample.
We can then perform algorithms on $\mathbf{Z}_I$ for learning.
We call this method the IND method.
To bound the sample error of this method, we can directly use the result in Section \ref{sec:pre}.

The key step of the IND method is to find a large $\mathbf{Z}_I$.
The larger $|\mathbf{Z}_I|$ is, the higher will be the expected accuracy of 
$f_{\mathbf{Z}_I}$.
If two hyperedges $e_a$ and $e_b$ in $G$ do not share any vertex,
i.e., $e_a^{(i)} \neq e_b^{(i)}$ for all $ 1\le i\le k$,
the two induced examples $\mathbf{z}_a$ and $\mathbf{z}_b$ are independent.
Therefore, finding a subset $\mathbf{Z}_I$ from the training dataset $\mathbf{Z}$,
is equivalent to finding an independent set in the dependency graph $\varGamma$, 
and is also equivalent to finding a hypergraph matching in $G$. 

For any dependency graph $\varGamma$, it holds that (see, e.g., \cite{diestel10}),
$$ \frac{m}{\chi^{*}(\varGamma)} \le \alpha(\varGamma) $$
where $\alpha$ is the independence number.
If we can find a maximum independent set of the dependency graph $\varGamma$, 
then the bound of the IND method will be better than that of the EQW method.

However, It is NP-hard to find a maximum independent set in $\varGamma$ or equivalently to find a maximum matching in $G$ when $k \ge 3$ \cite{garey79}.
Therefore, the IND method is not effective in practice since 
it is difficult to find a large independent set of networked examples.

\section{A weighting method}\label{sec:weighting}
In this section, we propose a computationally efficient method based on a weighting strategy.
It allows for a better bound of the sample error than the IND and EQW methods.

\subsection{Feasible weighting}
Given a hypergraph $G$, we weight every hyperedge $e_i$ with a nonnegative value $w_i$.
We use the notation $w_F$ to denote the sum $\sum_{i\in F} w_i$ over a set of indices $F \subseteq \{1,\ldots,n\}$ of hyperedges,
and $\eta(v)$ to denote the set of indices of hyperedges incident on a vertex $v \in V$.
We say that $\mathbf{w} = [w_1,\ldots, w_n]$ is a \emph{feasible weighting} of a hypergraph $G$ 
if for all $i$ it holds that $w_i\ge 0$ and for all $v\in V$ it holds that $w_{\eta(v)} \le 1$.  

For a hypergraph $G$, its $\mathsf{s}$-value is defined as follows:

\ifx\arxiv\undefined
\begin{align*}
\mathsf{s}(G) = &\\
\max_{\mathbf{w}} &\left\{ \sum_{e_i\in E} w_i : \mathbf{w} \text{ is a feasible weighting for }G  \right\}
\label{eq:s}
\end{align*}

\else
\begin{equation*}
\mathsf{s}(G) = \max_{\mathbf{w}} \left\{ \sum_{e_i\in E} w_i : \mathbf{w} \text{ is a feasible weighting for }G  \right\}
\label{eq:s} 
\end{equation*}
\fi

Notice that these constraints form a linear program on $\mathbf{w}$.
We call a $\mathbf{w}$ which makes the linear program maximal an \emph{optimal weighting}.
There exist efficient methods to solve the linear program formed by the above-mentioned constraints and hence 
compute an optimal weighting and the $\mathsf{s}$-value, 
e.g., interior point methods \cite{boyd04}.
An optimal weighting can be considered as a fractional maximum hypergraph matching \cite{lovasz75,chan12}.
One can show that the value $\mathsf{s}(G)$ is always greater than or equal to the size of a maximum hypergraph matching in any hypergraph $G$.

For a $G$-networked sample $\mathbf{Z}$, we denote the weighted sample $\mathbf{Z}_{\mathsf{s}} = \{(\mathbf{x}_i,y_i,w_i)\}$
where $[w_1,\ldots,w_n]$ is an optimal weighting.
Now we can define a new empirical risk on the weighted sample $\mathbf{Z}_{\mathsf{s}}$ by
$$\mathcal{E}_{\mathsf{s}}(f)=\frac{1}{\mathsf{s}}\sum_{i=1}^n w_i (f({\bf x}_i)-y_i)^2.$$
In the following, we will show the sample error bound of an ERM approach with $\mathbf{Z}_{\mathsf{s}}$.

\subsection{Exponential inequalities}\label{subsec:inequalities}
In Section \ref{sec:pre}, the Bernstein inequality 
is used to estimate the sample error.
A key property used for proving the Bernstein inequality is that all observations are independent.
That is, if $\xi_1,\ldots,\xi_m$ are independent random variables, then
$$\mathbf{E} \exp\left(\sum_{i=1}^m \xi_i\right) = \prod_{i=1}^m \mathbf{E} e^{\xi_i}.$$
However, when learning from networked training examples, the equality can not be used.

\ifx\arxiv\undefined
\else
In this section, we claim that, 
\begin{theorem}\label{thm:exponential}
Given a $G$-networked training dataset $\mathbf{Z}$ where $G=(V,E,{\cal X},{\cal Y},\phi)$,
and a function $\xi$ defined on ${\cal Z}$,
if $\mathbf{w}$ is a feasible weighting of $G$,
then
$$\mathbf{E}\exp\left(\sum_{i=1}^n{w_i \xi(\mathbf{z}_i)}\right) \le \prod_{i=1}^n \left(\mathbf{E} e^{\xi(\mathbf{z})}\right)^{w_i}.$$
\end{theorem}

Before we prove it, we give a lemma which will be used in the proof.
The lemma says that the weighted geometric mean function is concave.

\begin{lemma}\label{lem:concave}
If $\mathbf{\beta} = [\beta_1, \ldots, \beta_k] \in \mathbb{R}^k_+$ such that $\sum_{i=1}^k \beta_i \le 1$
and ${\bf t} = [t_1, \ldots, t_k] \in \mathbb{R}^k_+ $,
then the function $g(t) = \prod_{i=1}^k t_i^{\beta_i}$ is concave.
\end{lemma}
\begin{proof}
We prove by showing that its Hessian matrix $\nabla^2 g({\bf t})$ is negative semidefinite.
$\nabla^2 g({\bf t})$ is given by
$$\frac{\partial^2 g({\bf t})}{\partial t_i^2} = \frac{\beta_i(\beta_i-1)g({\bf t})}{t_i^2},
\qquad \frac{\partial^2 g({\bf t})}{\partial t_i \partial t_j} = \frac{\beta_i \beta_j g({\bf t})}{t_i t_j} ,$$
and can be expressed as
$$\nabla^2 g({\bf t}) = (qq^\mathrm{ T } - \textbf{diag}(\beta_1/t_1^2, \ldots, \beta_n/t_n^2))g({\bf t})$$
where $q=[q_1,\ldots,q_k]$ and $q_i = \beta_i/t_i$.
We must show that $\nabla^2 g({\bf t}) \preceq 0$, i.e., that
$$u^\mathrm{ T } \nabla^2 g({\bf t}) u =
\left( \left( \sum_{i=1}^k  \beta_i u_i/t_i \right)^2 - \sum_{i=1}^k \beta_i u_i^2/t_i^2 \right) g({\bf t}) \le 0$$
for all $u \in \mathbb{R}^k$.
We let $\beta_0 = 1 - \sum_{i=1}^k \beta_i, u_0 =0$ and $t_0$ be any positive number.
Because $g({\bf t}) \ge 0$ for all ${\bf t}$, we only need to prove
$$ \left( \sum_{i=0}^k  \beta_i u_i/t_i \right)^2 - \sum_{i=0}^k \beta_i u_i^2/t_i^2  \le 0.$$
This follows from the fact that the square function is convex, and $\sum_{i=0}^k  \beta_i u_i/t_i$ is a convex combination of $u_i/t_i$.
\end{proof}

Now, we can prove Theorem \ref{thm:exponential}.

\noindent{\bf Proof of Theorem \ref{thm:exponential}.}
First, we rewrite the left hand side as
$$
\mathbf{E}\exp\left(\sum_{i}{w_i \xi({\bf z}_i)}\right)
= \mathbf{E}_{x_1^{(1)},\ldots,x_k^{(n_k)}} \prod_{i}  \mathbf{E}_{y_i|{\bf x}_i}  \exp\left({w_i \xi({\bf z}_i)}\right)
$$
since these exponentials are independent given $x_1^{(1)},\ldots,x_k^{(n_k)}$.
Because $0 \le w_i\le 1$ for all $i$,
$$
\mathbf{E}\exp\left(\sum_{i}{w_i \xi({\bf z}_i)}\right)
\le \mathbf{E}_{x_1^{(1)},\ldots,x_k^{(n_k)}} \prod_{i} \left[\mathbf{E}_{y_i|{\bf x}_i} \exp\left({\xi({\bf z}_i)}\right)\right]^{w_i}.
$$
We calculate the expectation iteratively.
For $x_1^{(1)}$, we define $A = \{i | v_1^{(1)} \in e_i \}$ and $B = \{i | v_1^{(1)} \notin e_i \}$,
$$\mathbf{E}\exp\left(\sum_{i}{w_i \xi({\bf z}_i)}\right)\le \\
\mathbf{E}_{x_1^{(2)},\ldots,x_k^{(n_k)}} \left\{ \prod_{i\in B} \left[\mathbf{E}_{y_i|{\bf x}_i} \exp\left({\xi({\bf z}_i)}\right)\right]^{w_i}
\mathbf{E}_{x_1^{(1)}} \prod_{i\in A} \left[\mathbf{E}_{y_i|{\bf x}_i} \exp\left({\xi({\bf z}_i)}\right)\right]^{w_i}\right\}.
$$
By the definition of the feasible weighting, $\sum_{i\in A}\ w_i \le 1$.
According to Lemma \ref{lem:concave}, $\prod_{i\in A} \left[\mathbf{E}_{y_i|{\bf x}_i} \exp\left({\xi({\bf z}_i)}\right)\right]^{w_i}$
is a concave function given $x_1^{(2)},\ldots,x_k^{(n_k)}$.
Therefore,
$$\mathbf{E}\exp\left(\sum_{i}{w_i \xi({\bf z}_i)}\right)
\le \mathbf{E}_{x_1^{(2)},\ldots,x_k^{(n_k)}} \left\{ \prod_{i\in B} \left[\mathbf{E}_{y_i|{\bf x}_i} \exp\left({\xi({\bf z}_i)}\right)\right]^{w_i}
 \prod_{i\in A} \left[\mathbf{E}_{x_1^{(1)}}\mathbf{E}_{y_i|{\bf x}_i} \exp\left({\xi({\bf z}_i)}\right)\right]^{w_i}\right\}$$
which follows the Jensen's inequality.
Next, we repeat the steps above for $x_1^{(2)},\ldots,x_k^{(n_k)}$.
We can get
$$
\mathbf{E}\exp\left(\sum_{i}{w_i \xi({\bf z}_i)}\right) \le \prod_{i} \left(\mathbf{E} e^{\xi({\bf z})}\right)^{w_i}.
$$

Using Theorem \ref{thm:exponential}, we are able to obtain important inequalities used for estimating the sample error.
\fi
The following inequality is an analogue to the Bernstein inequality.
The inequality will be used later to estimate the sample error for a networked sample.

\begin{theorem}\label{thm:three.avg}
Let $\mathbf{Z}$ be a $G$-networked sample and $\xi$ be a function defined on the space ${\cal Z}$ with mean $\mathbf{E}(\xi) = \mu$,
variance $\sigma^2(\xi) = \sigma^2$, and satisfying $|\xi(\mathbf{z})-\mu|\le M$ for almost all $\mathbf{z} \in {\cal Z}$. 
If $\bf w$ is an optimal weighting of $G$, then for all $\epsilon > 0$,
\begin{align*}
& \Pr\left(\frac{1}{\mathsf{s}}\sum_{i} w_i \xi(\mathbf{z}_i) - \mu \ge \epsilon \right)
\le \exp \left( -\frac{\mathsf{s}\epsilon^2}{2(\sigma^2+\frac{1}{3}M\epsilon)} \right). \\
\end{align*}
\end{theorem}

\ifx\arxiv\undefined
\else
To prove Theorem \ref{thm:three.avg}, we first give necessary lemmas.
The main ideas were borrowed from \cite{cucker07}.

\begin{lemma}\label{lem:bennett}
Let $\mathbf{Z}$ be a $G$-networked sample and $\xi$ be a function defined on the space ${\cal Z}$ with mean $\mathbf{E}(\xi) = \mu$,
variance $\sigma^2(\xi) = \sigma^2$, and satisfying $|\xi(\mathbf{z})-\mu|\le M$ for almost all $\mathbf{z} \in {\cal Z}$. 
If $\bf w$ is an optimal weighting of $G$, then for all $\epsilon > 0$,
$$\Pr\left(\sum_i w_i (\xi({\mathbf{z}}_i)-{\mu}) \ge \epsilon \right) \le
\exp \left( -\frac{\mathsf{s}\sigma^2}{M^2} h\left(\frac{M\epsilon}{\mathsf{s}\sigma^2}\right) \right)$$
where $h$ is given by $h(a) = (1+a)\log(1+a)-a$.
\end{lemma}
\begin{proof}
Without loss of generality, we assume ${\mu} = 0$.
Let $c$ be an arbitrary positive constant which will be determined later.
Then
$$I := \Pr\left(\sum_{i=1}^N w_i \xi({\mathbf{z}}_i) \ge \epsilon \right) = \Pr\left(\exp\left(c\sum_{i=1}^N w_i \xi({\mathbf{z}}_i)\right) \ge e^{c\epsilon} \right).$$
By Markov's inequality and Theorem \ref{thm:exponential}, we have
$$I \le e^{-c\epsilon}\mathbf{E}\left(\exp\left(c\sum_{i=1}^N w_i \xi({\mathbf{z}}_i)\right)\right) \le  e^{-c\epsilon} \prod_i \left
(\mathbf{E}e^{c \xi({\mathbf{z}}_i)}\right)^{w_i}.$$
Since $|\xi({\mathbf{z}}_i)| \le M$ almost everywhere and ${\mu} = 0$, we have
$$\mathbf{E}e^{c \xi({\mathbf{z}}_i)} = 1 + \sum_{p=2}^{+\infty}\frac{c^p\mathbf{E}\xi^{p}({\mathbf{z}}_i)}{p!}\le 1+ \sum_{p=2}^{+\infty}\frac{c^pM^{p-2}\sigma^2}{p!}$$
from the Taylor expansion for exponential functions.
Using $1+a\le e^a$, it follows that
$$\mathbf{E}e^{c \xi({\mathbf{z}}_i)}  \le \exp\left( \sum_{p=2}^{+\infty}\frac{c^pM^{p-2}\sigma^2}{p!} \right) = \exp\left(\frac{e^{cM}-1-cM}{M^2}\sigma^2\right)$$
and therefore
$$I \le \exp\left(-c\epsilon + \frac{e^{cM}-1-cM}{M^2}\mathsf{s}\sigma^2 \right).$$

Now choose the constant $c$ to be the minimizer of the bound on the right hand side above:
$$ c= \frac{1}{M}\log(1+\frac{M\epsilon}{\mathsf{s}\sigma^2}).$$
That is, $e^{cM}-1=\frac{M\epsilon}{\mathsf{s}\sigma^2}$. With this choice,
$$I\le \exp \left( -\frac{\mathsf{s}\sigma^2}{M^2} h\left(\frac{M\epsilon}{\mathsf{s}\sigma^2}\right) \right).$$
This proves the desired inequality.

\end{proof}

\begin{lemma}\label{lem:three.sum}
Let $\mathbf{Z}$ be a $G$-networked sample and $\xi$ be a function defined on the space ${\cal Z}$ with mean $\mathbf{E}(\xi) = \mu$,
variance $\sigma^2(\xi) = \sigma^2$, and satisfying $|\xi(\mathbf{z})-\mu|\le M$ for almost all $\mathbf{z} \in {\cal Z}$. 
If $\bf w$ is an optimal weighting of $G$, then for all $\epsilon > 0$,
\begin{align*}
&\Pr\left(\sum_{i=1}^N w_i (\xi({\mathbf{z}}_i) - {\mu}) \ge \epsilon \right)
\le \exp \left( -\frac{\epsilon^2}{2(\sigma^2+\frac{1}{3}M\epsilon)} \right).\\
\end{align*}
\end{lemma}
\begin{proof}
The inequality follows from Lemma \ref{lem:bennett} and the inequality
$$h(a) \ge \frac{3a^2}{6+2a},\;\forall a\ge0. $$
%
\end{proof}
\fi

\subsection{An ERM approach with ${\bf Z}_{\mathsf{s}}$}\label{subsec:ermapproach}
In this section, we consider the ERM approach associated with ${\bf Z}_{\mathsf{s}}$. 
As discussed in section \ref{subsec:erm1}, 
the ERM approach aims to find a minimizer of the empirical risk in a proper hypothesis space $\mathcal{H}$ to approximate the target function, i.e.,
\begin{equation*}\label{def:fz}
f_{{\bf Z}_\mathsf{s}}=\arg\min_{f\in\mathcal{H}} \mathcal{E}_{\mathsf{s}}(f).
\end{equation*}
Recall the empirical risk with ${\bf Z}_{\mathsf{s}}$ takes the form 
$\mathcal{E}_{\mathsf{s}}(f)=\frac{1}{\mathsf{s}}\sum_{i=1}^n w_i(f({\bf x}_i)-y_i)^2,$
and the expected risk $\mathcal{E}(f)=\int (f({\bf x})-y)^2 \rho({\bf x},y) \hbox{d}({\bf x},y).$
As mentioned in section \ref{subsec:erm1}, the target function is the minimizer of the expected risk $\mathcal{E}(f)$,
one can easily see that the target function takes the form \cite{cucker07}
$$f_\rho(\mathbf{x})=\int y \rho_{y|\mathbf{x}}({\bf x}, y) \hbox{d}(\mathbf{x}, y).$$
Then the performance of the ERM approach is measured by the excess risk
$$\mathcal{E}(f_{{\bf Z}_\mathsf{s}})-\mathcal{E}(f_\rho).$$
Recall the definition $\label{def:fh}f_{\mathcal{H}}=\arg\min_{f\in\mathcal{H}} \mathcal{E}(f)$,
the excess risk can be divided into two parts (sample error and approximation error) as follows
$$\mathcal{E}(f_{{\bf Z}_\mathsf{s}})-\mathcal{E}(f_\rho)=[\mathcal{E}(f_{{\bf Z}_\mathsf{s}})-\mathcal{E}(f_{\mathcal{H}})]+[\mathcal{E}(f_{\mathcal{H}})-\mathcal{E}(f_\rho)].$$
Notice that the approximation error $\mathcal{E}(f_{\mathcal{H}})-\mathcal{E}(f_\rho)$ is independent of the sample ${\bf Z}_{\mathsf{s}}$,
and the approximation error vanishes if $f_\rho\in\mathcal{H}$.

In this section, we focus on the sample error
$\mathcal{E}_\mathcal{H}(f_{{\bf Z}_{\mathsf{s}}}):=\mathcal{E}(f_{{\bf Z}_{\mathsf{s}}})-\mathcal{E}(f_\mathcal{H})$.
\ifx\arxiv\undefined
\else
To this end, we use the inequalities with $\mathsf{s}$-value in section \ref{subsec:inequalities}
to estimate the sample error $\mathcal{E}_\mathcal{H}(f_{{\bf Z}_{\mathsf{s}}}).$

We assemble some lemmas,
which will be used to establish the sample error bounds for the ERM algorithm associated with the networked training sample.

Denote the defect function $\mathcal{D}_{{Z}_{\mathsf{s}}}(f)=\mathcal{E}(f)-\mathcal{E}_{{Z}_{\mathsf{s}}}(f).$
Then the following lemma follows directly from the third inequality in Theorem \ref{thm:three.avg} by taking $\xi=-(f(x)-y)^2$
satisfying $|\xi|\le M^2$ when $f$ is M-bounded.
\begin{lemma}\label{lem:single}
Let $M > 0$ and $f : \mathcal{X} \mapsto \mathcal{Y}$ be $M$-bounded. Then for all $\epsilon > 0$,
$$ \Pr\left( \mathcal{D}_{{Z}_{\mathsf{s}}}(f) \ge -\epsilon \right) \ge 1-\exp\left( \frac{\mathsf{s} \epsilon^2}{2M^4} \right). $$
\end{lemma}

\begin{lemma}\label{lem:familyclass1}
Let $\mathcal{H}$ be a compact M-bounded subset of $\mathcal{C}(\mathcal{X}).$ Then, for all $\epsilon>0,$
$${\Pr}\Big(\sup_{f\in\mathcal{H}}\mathcal{D}_{{Z}_{\mathsf{s}}}(f)\leq \epsilon\Big)\ge 1-\mathcal{N}\Big(\mathcal{H},\frac{\epsilon}{8M}\Big) \exp\Big(-\frac{\mathsf{s}\epsilon^2}{8M^4}\Big).$$
\end{lemma}
\begin{proof} Let $\{f_j\}_{j=1}^\ell\subset \mathcal{H}$ with $\ell=\mathcal{N}\Big(\mathcal{H},\frac{\epsilon}{4M}\Big)$ such that $\mathcal{H}$ is covered by disks $D_j$ centered at $f_j$ with radius $\frac{\epsilon}{4M}.$ Let $U$ be a full measure set on which $\sup_{f\in\mathcal{H}}|f(x)-y|\le M$. Then for all ${Z}\in U^n$ and for all $f\in D_j$, there holds
$$|\mathcal{D}_{{Z}_{\mathsf{s}}}(f)-\mathcal{D}_{{Z}_{\mathsf{s}}}(f_j)|\leq 4M\|f-f_j\|_\infty\leq 4M \frac{\epsilon}{4M}=\epsilon.$$
Consequently,
$$\sup_{f\in D_j} \mathcal{D}_{{Z}_{\mathsf{s}}}(f)\ge 2\epsilon \Rightarrow \mathcal{D}_{{Z}_{\mathsf{s}}}(f_j)\ge \epsilon.$$
Then we conclude that, for $j=1,\cdots,\ell,$
$${\Pr}\Big\{\sup_{f\in {D_j}} \mathcal{D}_{{Z}_{\mathsf{s}}}(f)\ge 2\epsilon\Big\}\leq {\Pr}\Big\{ \mathcal{D}_{{Z}_{\mathsf{s}}}(f_j)\ge 2\epsilon\Big\}\leq \exp\Big\{-\frac{\mathsf{s}\epsilon^2}{2M^4}\Big\}$$ the last inequality follows by taking $\xi=-(f(x)-y)^2$ on $Z$. In addition, one can easily see that
$$\sup_{f\in\mathcal{H}} \mathcal{D}_{{Z}_{\mathsf{s}}}(f)\ge \epsilon \Leftrightarrow \exists j\le \ell ~s.t. ~\sup_{f\in D_j} \mathcal{D}_{{Z}_{\mathsf{s}}}(f)\ge \epsilon$$
and the fact that the probability of a union of events is bounded by the sum of the probabilities of these events. Hence
$${\Pr}\Big(\sup_{f\in\mathcal{H}}\mathcal{D}_{{Z}_{\mathsf{s}}}(f)\ge \epsilon\Big)\leq \sum_{j=1}^\ell {\Pr}\Big(\sup_{f\in D_j}\mathcal{D}_{{Z}_{\mathsf{s}}}(f)\ge \epsilon\Big) \leq \ell \exp\Big(-\frac{\mathsf{s}\epsilon^2}{8M^4}\Big).$$
This completes our proof.
\end{proof}

\begin{lemma}\label{lem:ratiosingle}
Suppose a random variable $\xi$ on $\mathcal{Z}$ satisfies $\mathbf{E}(\xi)=\mu\ge 0,$ an $|\xi-\mu|\leq B$ almost everywhere. If $\mathbf{E}(\xi^2)\leq c\mathbf{E}(\xi),$ then for every $\epsilon>0$ and $0<\alpha\leq 1,$ there holds
$${\Pr}_\mathcal{Z} \Big\{\frac{\mu-\frac{1}{\mathsf{s}}\sum_{i=1}^N w_i \xi(z_i)}{\sqrt{\mu+\epsilon}}>\alpha \sqrt{\epsilon}\Big\}\leq \exp\Big\{-\frac{\alpha^2 \mathsf{s} \epsilon}{2c+\frac{2}{3}B}\Big\}.$$
\end{lemma}
\begin{proof}
The lemma follows directly from the second statement of Theorem \ref{thm:three.avg}.

\end{proof}
Lemma \ref{lem:ratiosingle} can also be extended to families of functions as follows.
\begin{lemma}\label{lem:ratiofamilyclass}
Let $\mathcal{G}$ be a set of functions on $\mathcal{Z}$ and $c>0$ such that, for each $g\in\mathcal{G},$ $\mathbf{E}(g)\ge 0$, 
$\mathbf{E}(g^2)\leq c\mathbf{E}(g)$ and $|g-\mathbf{E}(g)|\leq B$ almost everywhere. 
Then for every $\epsilon>0$ and $0<\alpha\le 1,$ we have
$${\Pr}_\mathcal{Z} \Big\{\sup_{g\in \mathcal{G}} \frac{\mathbf{E}(g)-\mathbf{E}_{{Z}_{\mathsf{s}}}(g)}{\sqrt{\mathbf{E}(g)+\epsilon}} > 4 \alpha \sqrt{\epsilon}\Big\}\leq \mathcal{N}(\mathcal{G},\alpha \epsilon)\exp\Big\{-\frac{\alpha^2 \mathsf{s} \epsilon}{2c+\frac{2}{3}B}\Big\}.$$
\end{lemma}
\begin{proof}
The proof of this lemma is similar to that of Lemma \ref{lem:familyclass1}.
Details for the proof can also be found in Chapter 3 of \cite{cucker07}.
\end{proof}

Denote $\mathcal{L}_\rho^2(\mathcal{X})$ as a Banach space with the norm
$\|f\|_{{\mathcal{L}}_\rho^2({\mathcal X})}=\Big(\int_\mathcal{X}|f(x)^2|d\rho_{\bf x}\Big)^\frac{1}{2}.$

\begin{lemma}\label{lem:convex}{\rm \cite{cucker07}}
Let $\mathcal{H}$ be a convex subset of $\mathcal{C}(\mathcal{X})$ such that $f_\mathcal{H}$ exists. 
Then $f_\mathcal{H}$ is unique as an element in $\mathcal{L}_\rho^2(\mathcal{X})$ and for all 
$f\in \mathcal{H},$ $$\int_{\cal X}(f_\mathcal{H}(x)-f(x))^2 \rho(x) dx \leq \mathcal{E}_\mathcal{H}(f).$$
In particular, if $\rho(x)$ is not degenerate then $f_\mathcal{H}$ is unique in $\mathcal{H}.$
\end{lemma}

\fi

The following is our main result.
\begin{theorem} 
Let $\mathcal{H}$ be a compact and convex subset of $\mathcal{C}(\mathcal{X}).$ 
If $\mathcal{H}$ is a M-bounded, 
then for all $\epsilon>0$,
$$\Pr\big(\mathcal{E}_\mathcal{H}(f_{{\bf Z}_{\mathsf{s}}})\ge \epsilon\big)\le \mathcal{N}\Big(\mathcal{H},\frac{\epsilon}{12M}\Big)  
\exp\Big(-\frac{\mathsf{s}\epsilon}{300M^4}\Big).$$
\end{theorem}
\noindent{\bf Remark:}
In this paper, we mainly consider the ERM algorithm associated with networked samples to avoid over-fitting.
Another way to deal with over-fitting is regularization, which is initially proposed to solve ill-posed phenomena induced in inverse problems, 
e.g., ill-conditioned matrix inversion problems.
Similar results can also be obtained for the regularization algorithms by using the probability inequalities in section \ref{subsec:inequalities}.

\section{Conclusions}\label{sec:conclusion}
In this paper, we introduce the problem of learning from networked data. 
We first show that this may result in a poor sample error bound if we ignore the dependency relationship between the examples.
We then analyze a method where first a set of i.i.d. examples is selected.
Existing theoretical results can be directly used for this method, but it is difficult to find a large set of independent examples.
We propose a novel method which is a weighting strategy with efficiently computable weights.
To assess learning algorithms on these weighted examples, we show a Bernstein-type statistical inequality.
Using this inequalitiy, we can estimate the sample error.  
We show that this bound is better than existing alternatives.

In future, we want to consider settings where we do not make the strong independence assumption that the occurrences
of the hyperedges are independent of the features of the vertices.
A first step in this direction would be to develop a measure to assess the strength of the dependency of the hyperedges on the features of the vertices
and its influence on the learning task at hand.

\subsubsection*{Acknowledgements}
The first author and the second author are supported by ERC Starting Grant 240186
“MiGraNT: Mining Graphs and Networks: a Theory-based approach”.
The third author is supported by the EPSRC under grant EP/J001384/1.

\bibliography{cap2013}

\end{document}